\documentclass{article}

\usepackage{microtype}
\usepackage{graphicx}
\usepackage{float}
\usepackage{caption,subcaption}
\usepackage{booktabs}
\usepackage{hyperref}

\usepackage[accepted]{icml2024}

\usepackage{amsmath}
\usepackage{amssymb}
\usepackage{mathtools}
\usepackage{amsthm}

\def\eqref#1{equation~\ref{#1}}
\def\1{\bm{1}}

\DeclareMathAlphabet{\mathsfit}{\encodingdefault}{\sfdefault}{m}{sl}
\SetMathAlphabet{\mathsfit}{bold}{\encodingdefault}{\sfdefault}{bx}{n}

\newcommand{\R}{\mathbb{R}}

\DeclareMathOperator*{\argmin}{arg\,min}

\newcommand{\norm}[1]{\left\lVert#1\right\rVert}
\newcommand{\paren}[1]{\left ( #1\right)}
\newcommand{\set}[1]{\left \{ #1\right\}}
\newcommand{\abs}[1]{\left\lvert#1\right\rvert}
\def\R{\mathbb{R}}
\def\bv{\mathbf{v}}

\def\bbN{\mathbb{N}}

\def\cG{\mathcal{G}}
\def\cM{\mathcal{M}}

\def\cX{\mathcal{X}}
\usepackage[capitalize,noabbrev]{cleveref}

\theoremstyle{plain}
\newtheorem{theorem}{Theorem}[section]
\newtheorem{proposition}[theorem]{Proposition}
\newtheorem{lemma}[theorem]{Lemma}

\theoremstyle{definition}
\newtheorem{definition}[theorem]{Definition}

\theoremstyle{remark}

\usepackage[textsize=tiny]{todonotes}

\icmltitlerunning{(Deep) Generative Geodesics}

\begin{document}

\twocolumn[
\icmltitle{(Deep) Generative Geodesics}

\icmlsetsymbol{equal}{*}

\begin{icmlauthorlist}
\icmlauthor{Beomsu Kim}{equal,yyy}
\icmlauthor{Michael Puthawala}{equal,comp}
\icmlauthor{Jong Chul Ye}{yyy}
\icmlauthor{Emanuele Sansone}{sch}
\end{icmlauthorlist}

\icmlaffiliation{yyy}{Kim Jaechul Graduate School of AI, KAIST}
\icmlaffiliation{comp}{South Dakota State University}
\icmlaffiliation{sch}{KU Leuven}

\icmlcorrespondingauthor{Michael Puthawala}{michael.puthawala@sdstate.edu}
\icmlcorrespondingauthor{Beomsu Kim}{beomsu.kim@kaist.ac.kr}
\icmlcorrespondingauthor{Jong Chul Ye}{jong.ye@kaist.ac.kr}
\icmlcorrespondingauthor{Emanuele Sansone}{emanuele.sansone@kuleuven.be}

\icmlkeywords{Deep generative models, Riemannian geometry, Differential geometry}

\vskip 0.3in
]

\printAffiliationsAndNotice{\icmlEqualContribution}

\begin{abstract}
In this work, we propose to study the global geometrical properties of generative models. We introduce a new Riemannian metric to assess the similarity between any two data points. Importantly, our metric is agnostic to the parametrization of the generative model and requires only the evaluation of its data likelihood. Moreover, the metric leads to the conceptual definition of generative distances and generative geodesics, whose computation can be done efficiently in the data space. Their approximations are proven to converge to their true values under mild conditions. We showcase three proof-of-concept applications of this global metric, including clustering, data visualization, and data interpolation, thus providing new tools to support the geometrical understanding of generative models.
\end{abstract}

\section{Introduction}
What does it mean for two data points to be similar, what is the correct notion of similarity, and how can we measure similarity with generative models? Deep generative models work by effectively learning a generative map from a `simple' latent distribution, typically a Gaussian distribution in low-dimensional Euclidean space, to an empirical distribution in data space, called the data distribution.

Several works have locally studied the geometry of the data distribution by using the notion of the pullback metric, which requires access to the latent representation of a deep generative model and the Jacobian of the generative map. This enables (i) the control and manipulation of the generation process in the latent space~\citep{arvanitidis2018latent,chen2018metrics,shao2018riemannian,arvanitidis2021geometrically,arvanitidis2022prior,lee2022regularized,lee2022statistical,lee2023explicit,issenbuth2023unveiling,ramesh2019spectral,zhu2021low,choi2022do,park2023understanding,song2023latent}, and (ii) the investigation of the behavior of deep generative models, for instance through latent interpolation~\citep{berthelot2019inter,zhu2020vp,shao2018riemannian,michelis2021linear,arvanitidis2018latent,chen2018metrics,laine2018feature,struski2023feature}.

In this work, we go beyond the study of the local geometry of the data distribution and instead propose a global metric, which is agnostic to the internal parametrization of the deep generative model and only requires the evaluation of its data likelihood. This allows the introduction of two new concepts: \textit{generative distances} and \textit{generative geodesics}. Generative distances enable to assess the similarity between any two points in the data space according to the generative model, whereas generative geodesics identify the corresponding connecting path. Notably, we show that generative distances and geodesics can be efficiently approximated by first discretizing the ambient space using a graph constructed from training and/or synthetic data, by second attaching weights to this graph based on the proposed metric and then using shortest-path algorithms over the constructed weighted graph to overcome the curse of dimensionality. Moreover, we prove that this approximation converges to the true value under mild conditions on the data distribution. 

In summary, our key contributions are: (1) We conceptualize a new Riemannian metric that we call \textit{generative distances} and \textit{generative geodesics} along with approximation and convergence results, \S\ref{sec:generative}. (2) We combine recent advancements in geodesic computation and classical graph theory to efficiently compute our metric, \S\ref{sec:deepgenerative:approx-and-conv}. (3) We highlight the global nature of the new proposed metric by showcasing three proof-of-concept applications of our theory, including clustering, data visualization and data interpolation, \S\ref{sec:experiments}.

\subsection{Background on Geodesics and Metrics}
\label{sec:preliminaries}
Geodesics generalize the notion of shortest-path distance to spaces without straight-line distances or non-Euclidean distance.
\begin{figure}
    \centering
    \includegraphics[width=.6\linewidth]{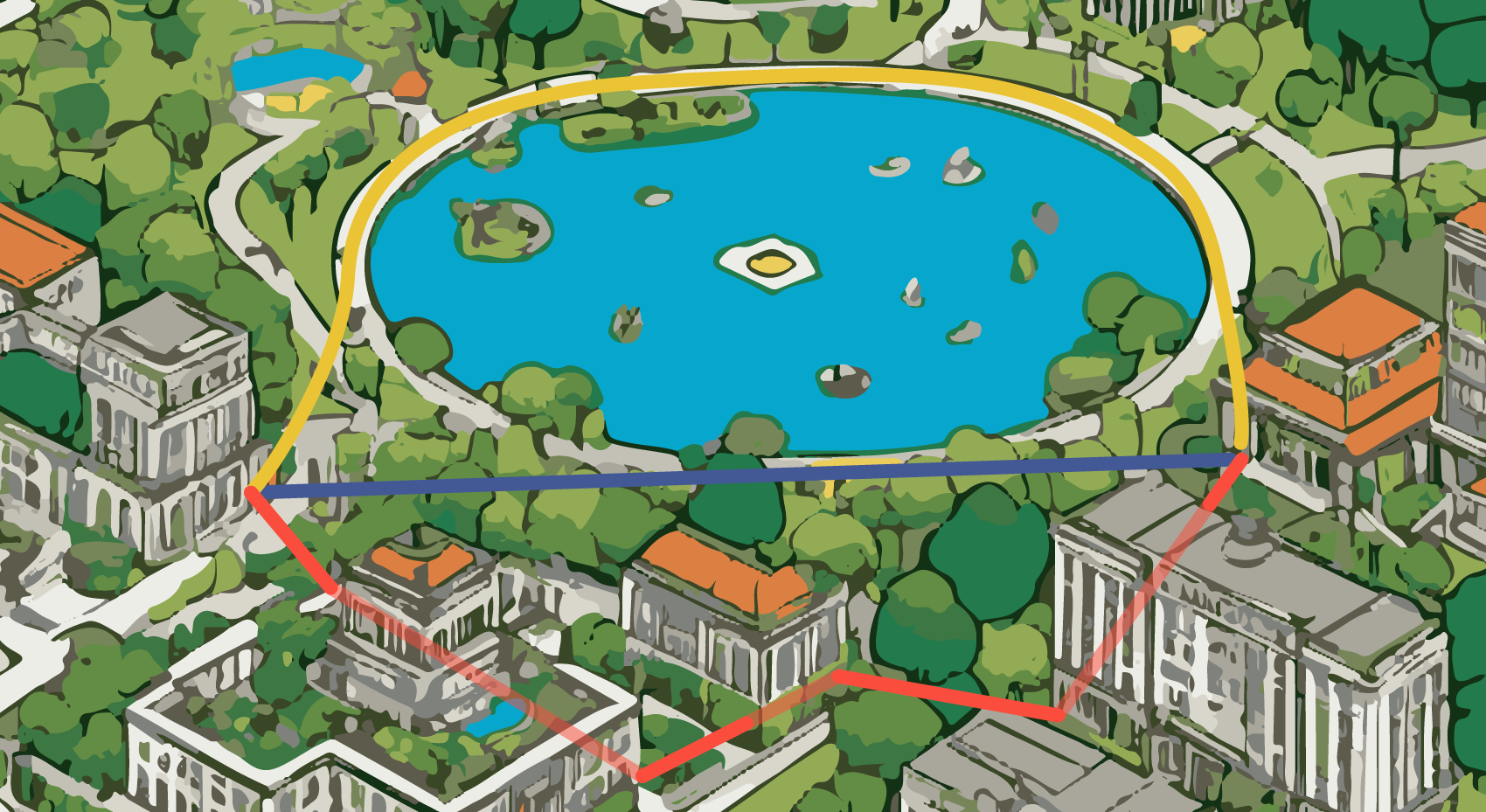}
    \caption{A visualization of how different paths can be optimal under different preferences. The blue path has the least walking, the yellow path is the most scenic, and the red path has the least sun exposure. Image generated with Stable Diffusion model.}
    \label{fig:varrying-path-fig}
\end{figure}
Suppose that you have to walk from one side of a town to the other, as in Fig. \ref{fig:varrying-path-fig}, and are considering the path that you'd like to take. The shortest path depends on your preferences, costs, budget, namely on the underlying model you consider (in our case a generative one). %would be a straight line. If, however, you wanted to take a less direct but more scenic path, than a more optimal path might pass by a lake on the way to your destination. If you had sensitive skin and wanted to avoid sun exposure, than the optimal path would pass through indoor and shaded areas instead. The point is that what makes a path optimal depends on your preferences, costs, budget, etc. These can be mathematically expressed by placing a metric on the space.

We now formalize this notion, and recall notations from differential geometry. For a detailed definition of the terms used, see e.g. Chap. 13 \cite{lee2013smooth} and Chaps. 5, 6 \cite{lee2006riemannian}. Let $\cM$ be a smooth manifold with tangent space $T_p\cM$ for point $p \in \cM$. A metric is a function $g_p \colon T_p\cM \times T_p\cM \to \R$ that is symmetric and positive definite. Given this, we define the $g$ norm and length of a smooth curve $\gamma  \colon [0, 1] \to \cM$ by
\begin{align}
    \norm{u}_g \coloneqq \sqrt{g_p(u,u)},\quad 
    \label{eqn:length-def}
    L_g(\gamma) \coloneqq \int_{0}^1 \norm{\gamma'(t)}_g dt.
\end{align}
A non-obvious but important fact, is that the notion of length given in Eqn. \ref{eqn:length-def} is parameterization independent. Put differently, $L_g(\gamma)$ measures the length of the path traced out by $\gamma$, not a property of how $\gamma$ is parameterized.

For any two $p,q \in \cM$ on the same connected component of a manifold, we may define the \emph{Riemannian distance from $p$ to $q$ induced by $g$} by
\begin{align}
    \label{eqn:minimal-distance}
    d_g(p,q) \coloneqq \inf_{\gamma \in \Omega(p,q)} L_g(\gamma)
\end{align}
where $\Omega(p,q)$ is the set of all paths that connect $p$ to $q$. A length minimizing path, that is an $\argmin$ of the r.h.s. of Eqn. \ref{eqn:minimal-distance}, is a \emph{geodesic}.

\section{Our Contribution}
\label{sec:our-contribution}

{\subsection{Generative Geodesics}\label{sec:generative}}

\begin{figure}
    \centering
    \begin{subfigure}{.65\linewidth}
        \centering
        \includegraphics[width=\linewidth]{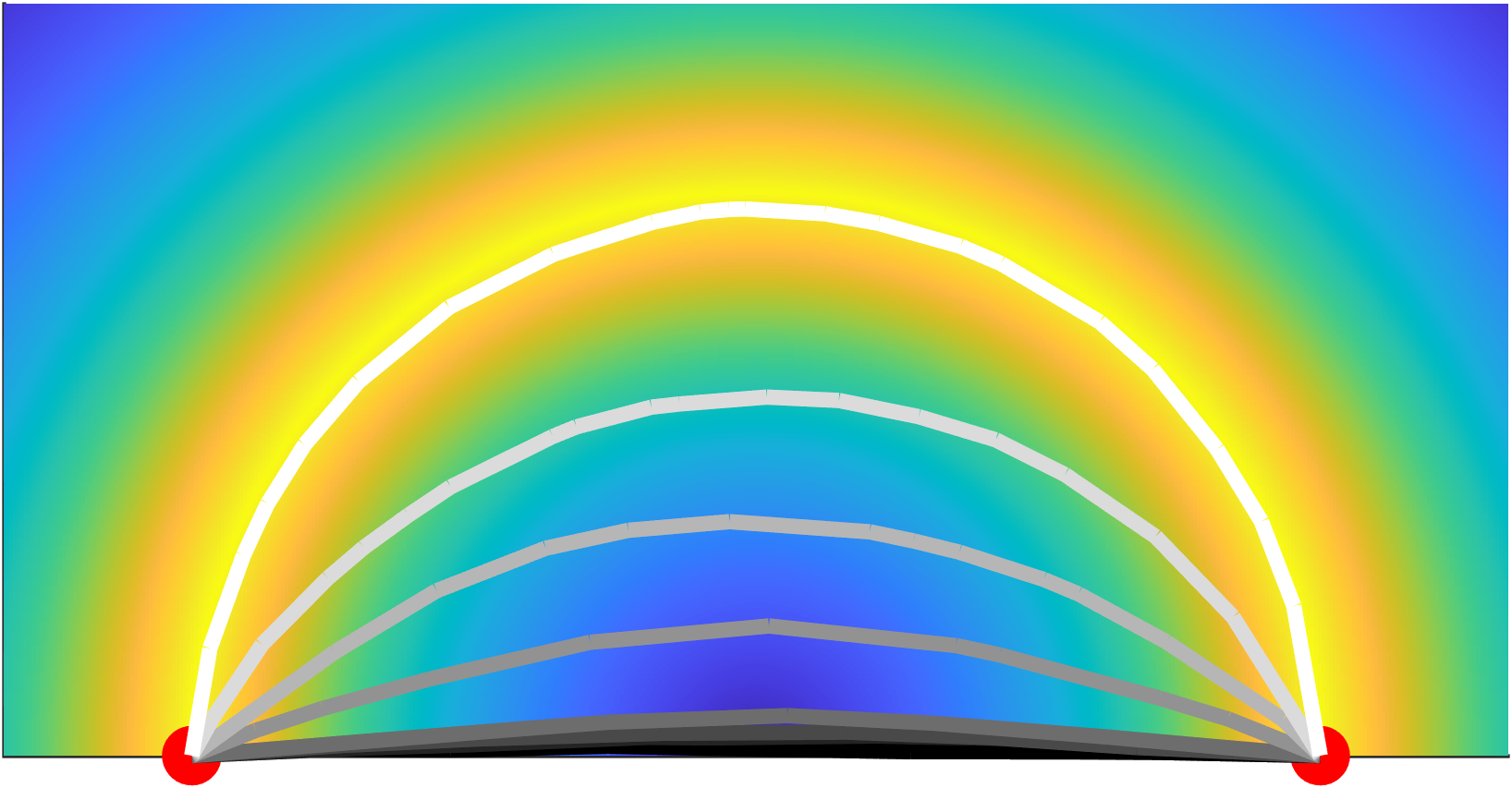}
    \end{subfigure}
    \caption{A showcase of how decreasing $\lambda$ changes the shortest path between the two red points. The black curve has $\lambda = 5$, the white curve $\lambda = 0.05$, and the other curves interpolate between the two. The background $p_\Psi(x)$ probability distribution is shown with a yellow to blue gradient with yellow indicating high probability and blue low probability.}
    \label{fig:comparison-of-parameters:a}
\end{figure}
Let $\Omega \subset \R^n$, $p_\Psi \colon \Omega \to \R^+$ be pointwise non-negative density function on $\Omega$, and $\lambda \geq 0$ and $p_0 >0$ be given.
We may define the \emph{generative metric} $g_{x,\Psi,\lambda}$ on $\Omega$ at point $x$ by 
\begin{align}
    \label{eqn:g-lambda-def}
    g_{x,\Psi,\lambda}(u,v) = \paren{\frac{p_0 + \lambda}{p_\Psi(x) + \lambda}}^2u \cdot v
\end{align}
where $\cdot$ denotes the Euclidean dot product. Define the Riemannian metric on $\Omega$, denoted $g_{\Psi,\lambda}$, as the flat metric defined point-wise on $g_{x,\Psi,\lambda}$. From this define the \emph{generative geodesic} $d_{\Psi,\lambda}$ on $\Omega$ using Eqn.s \ref{eqn:length-def} and \ref{eqn:minimal-distance}.

\begin{lemma}[$d_{\Psi,\lambda}$ Metric]
    \label{lem:d-lambda-metric-basic-props}
    Let $\lambda, p_0 > 0 $, and $p_\Psi$ be a smooth probability density on $\R^n$. Then $d_{\Psi,\lambda}$ is a Riemannian metric on $\R^n$.
    Further, for $x,y \in \R^n$,
    \begin{enumerate}
        \item as $\lambda \to \infty$, $d_{\Psi,\lambda}(x,y)$ converges to $d_2(x,y)$. That is $d_{\Psi,\lambda}$ converges to the Euclidean distance. % That is, $\lim_{\lambda\to\infty}d_{\Psi,\lambda}(x,y) \to d_2(x,y)$.
        \item Let $p_\Psi(x) > 0$ for all $x$, and define $\tilde d$ to be the distance induced by the metric $\frac{p_0}{p_{\Psi}(x)}u \cdot v$. As $\lambda \to 0$, $d_{\Psi,\lambda} \to \tilde d$.% converges to the distance induced by the metric that is pointwise given by $g_{\Psi,0}(u,v) \coloneqq \frac{p_0}{p_{\Psi}(x)}u \cdot v$.
    \end{enumerate}
    Moreover, if $\Omega$ is bounded, then this convergence is uniform (i.e. doesn't depend on $x,y$).
\end{lemma}
The lemma allows us to understand the role and effects of changing $\lambda$.
If $\lambda$ is large relative to $p_0$ and $p_\Psi$, then ${\frac{p_0 + \lambda}{p_\Psi(x) + \lambda}} \approx 1$, and $d_{\Psi,\lambda}$ reduces to the Euclidean distance. As $\lambda$ decreases, $d_{\Psi,\lambda}$ tends to consider points to be close if they are connected via a path of high likelihood, and far apart otherwise. This idea is illustrated through an example in Figure \ref{fig:comparison-of-parameters:a}. If $p_\Psi$ is the likelihood associated with a generative model, then being connected via a path of high likelihood implies that the two points lie in the same connected components of the data manifold. It is for this reason that we refer to it as the generative geodesic.
\begin{definition}[Generative Geodesic]
    \label{def:generative-geodesic}
    For any $\lambda > 0$ and $x,y \in \Omega$ we call $d_{\Psi,\lambda}(x,y)$ the \emph{generative distance} from $x$ to $y$, and we call a $d_{\Psi,\lambda}$ length-minimizing $\gamma^* \in \Omega(x,y)$ a \emph{generative geodesic}.
\end{definition}

{\subsection{Approximation and Convergence Guarantee}
\label{sec:deepgenerative:approx-and-conv}}

Our metric can be readily approximated using the graph-theory approach given by \cite{davis2019approximating}. The idea is to first form an epsilon graph from the points, then give the graph edge weights by using a quadrature for Eqn. \ref{eqn:length-def}, then use a shortest path finding algorithm. We denote by $\cX$ the set of data, either training or synthetic. Define the $\epsilon$ graph, notated $\cG_\epsilon(\cX)$, as the graph formed by drawing edges between all pairs in $\cX$ less than $\epsilon$ apart. For a visualization of $\cG$ on two-dimensional domain, see Figure \ref{fig:eps-graph}.

\begin{definition}[$\epsilon$ Graph]
    \label{def:eps-graph}
    Let $\epsilon > 0$ and $\cX \coloneqq \set{x^{(i)}}_{i \in I}$ be given. The \emph{$\epsilon$ graph}, denoted by $\cG_\epsilon(\cX)$, is defined as the graph with vertices $V \coloneqq \cX$ and edges
    \begin{align*}
        E_\epsilon \coloneqq \set{(x,x') \colon \text{where } x,x'\in \cX \text{ are s.t. } \norm{x - x'} < \epsilon }.
    \end{align*}
\end{definition}
\begin{figure}
    \centering
    \begin{subfigure}{.38\linewidth}
        \centering
        \includegraphics[width=1.\linewidth]{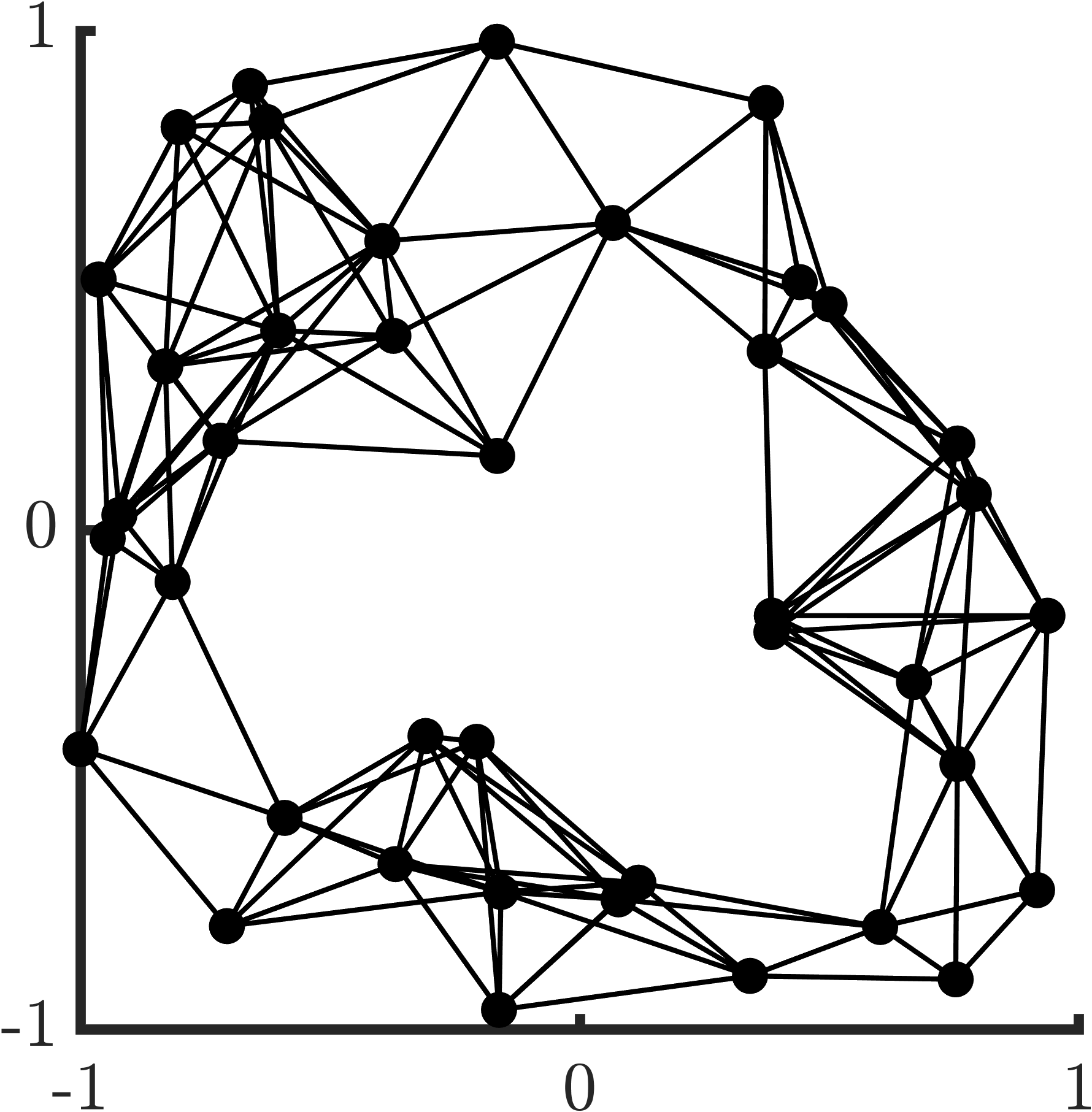}
        \subcaption{Plot of $\cG_{0.56}(\cX)$}
        \label{fig:eps-graph:2}
    \end{subfigure}
    \begin{subfigure}{.38\linewidth}
        \centering
        \includegraphics[width=1.\linewidth]{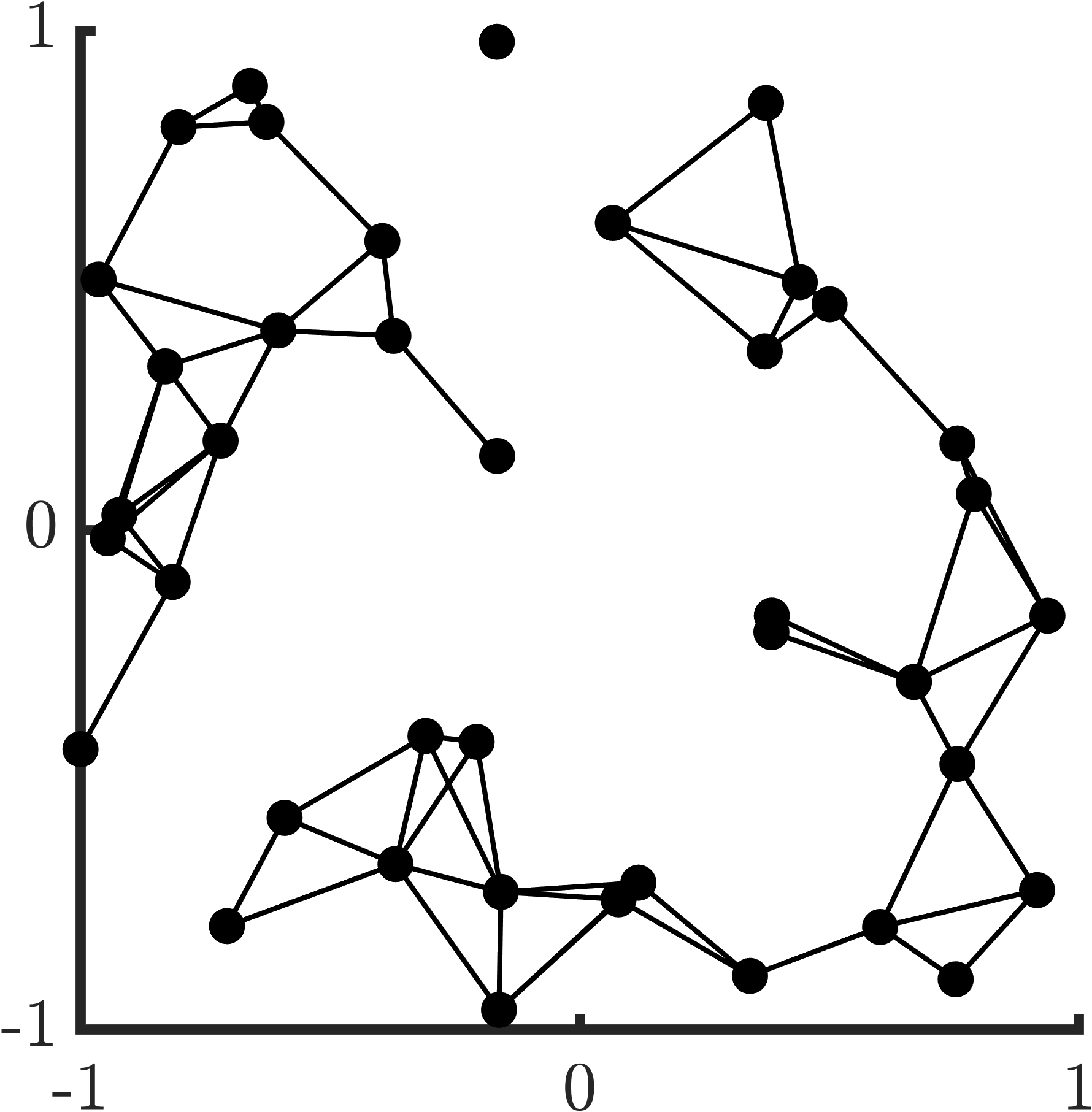}
        \subcaption{Plot of $\cG_{0.4}(\cX)$}
        \label{fig:eps-graph:3}
    \end{subfigure}
    \vspace{-1em}
    \caption{A figure demonstrating the formation of the graph $\cG_\epsilon(\cX)$, for different choices of $\epsilon$. In both figures, the black dots denote points in $\cX$, and black lines denote the edges in $E_\epsilon$. Notice how when $\epsilon = 0.56$, as in \ref{fig:eps-graph:2}, the $\cG$ has one connected component, whereas when $\epsilon = 0.4$, as in \ref{fig:eps-graph:3}, $\cG$ has several disconnected components.% edges that connect them.
    }
    \label{fig:eps-graph}
\end{figure}

\begin{definition}[Weighted $\epsilon$ graph]
    Let $\cG_{\epsilon}(\cX)$, $\Psi$, $p_0$ and $\lambda$ be given. We call the weighted $\epsilon$ graph, denoted by $\hat{ \mathcal{G}}_{\epsilon}(\cX)$, the graph $\mathcal{G}_{\epsilon}(\cX)$ 
    with edge weights given by $L_{g_{\Psi,\lambda}}(\ell_{x,y})$ where $(x,y) \in E_\epsilon$ $\ell_{x,y}$ is a straight line from $x$ to $y$.
 \end{definition}
 
    Similarly, we define the \emph{$K$-approximate weighted $\epsilon$ graph}, denoted by $\hat \cG_{\epsilon,K}(\cX)$, as the graph $\cG_{\epsilon}(\cX)$ with edge weights given by applying an integral quadrature\footnote{We use equispaced trapezoidal rule for our numerical experiments.} to $L_{g_{\Psi,\lambda}}(\ell_{x,y})$.
    
    Let $\hat \cG$ be an edge weighted graph with vertices $\cX$. Given two vertices $x,y \in  \hat \cG$ define the \emph{linear interpolating cost}, denoted $\hat L(x,y; \hat \cG)$, as the cost of the minimal path connecting $x$ to $y$ in $\hat \cG$. If $x$ and $y$ are not connected in $\hat \cG$, then define $\hat L_{n,K}(x,y; \hat \cG) = \infty$. The idea is that we may approximate $d_{\Psi,\lambda}(x,y)$ by computing $\hat L_{n,K}(x,y; \hat \cG)$. 
The following theorem is the main theoretical contribution of this paper and says that this approximation converges as the number of points in $\cX$ increases and $\epsilon$ decreases\footnote{This point of $\epsilon$ decreasing is somewhat technical, and not elaborated upon in this manuscript. It is summarized in Eqn. 2.6 \cite{davis2019approximating} and the surrounding discussion.}. This is true not only in value, but also in minimizing path, as the following theorem shows. Its proof is in Appendix \ref{sec:thm:conv-of-k-approx-lin-interp-cost-to-geodesic-cost}.
\begin{theorem}[Convergence of Linear Interpolation Costs to Riemannian Distance]
    \label{thm:conv-of-k-approx-lin-interp-cost-to-geodesic-cost}
    Let $\mathcal X_n$ be $n$ realizations 
    of a uniform random variable $\textsc{x}$ over $\Omega$. Then for any $x,y \in \Omega$ sequence $\paren{K_n}_{n = 1,\dots}$ so that $\lim_{n \to \infty}K_n = \infty$, and $\epsilon_n \to 0$ sufficiently quickly. Then the following results hold with probability 1 (over $\textsc{x}$). %\map{It's likely that the assumption taking $K \to \infty$ can be removed here.}
    \begin{enumerate}
        \item \textbf{Graph Becomes Connected} There is an $N$ so that for all $n \geq N$, there is a path from $x$ to $y$ in $\hat{\cG}_{\epsilon_n,K_n}(\cX_n \cup \set{x,y})$.
        \item \textbf{Convergence in Length} $\lim_{n{\to} \infty} \hat L_{n,K_n}(x,y;\hat{\cG}_{\epsilon_n{,}K_n}(\cX_n{\cup}\set{x{,}y})) {=}d_{\Psi{,}\lambda}(x{,}y)$.
        \item \textbf{Convergence in Geodesic} For any sequence of optimal paths $l_{\bv^{(n)}} \in \argmin \hat L_{n,K_n}(x,y)$, any subsequence of $\paren{l_{\bv^{(n)}}}_{n = 1,\dots}$ has a further subsequence of linear paths that converges uniformly to a limit path $\gamma \in \argmin L$ and of the discrete paths in the Hausdorff sense to the true generative geodesic, $S_\gamma$. Further, if the path is unique (up to reparameterization), then $\paren{l_{\bv^{(n)}}}_{n = 1,\dots}$ converges in Hausdorff distance to $S_\gamma$.
    \end{enumerate}
\end{theorem}

\begin{figure}
    \centering
    \begin{subfigure}{.98\linewidth}
        \centering
        \includegraphics[width=.48\linewidth]{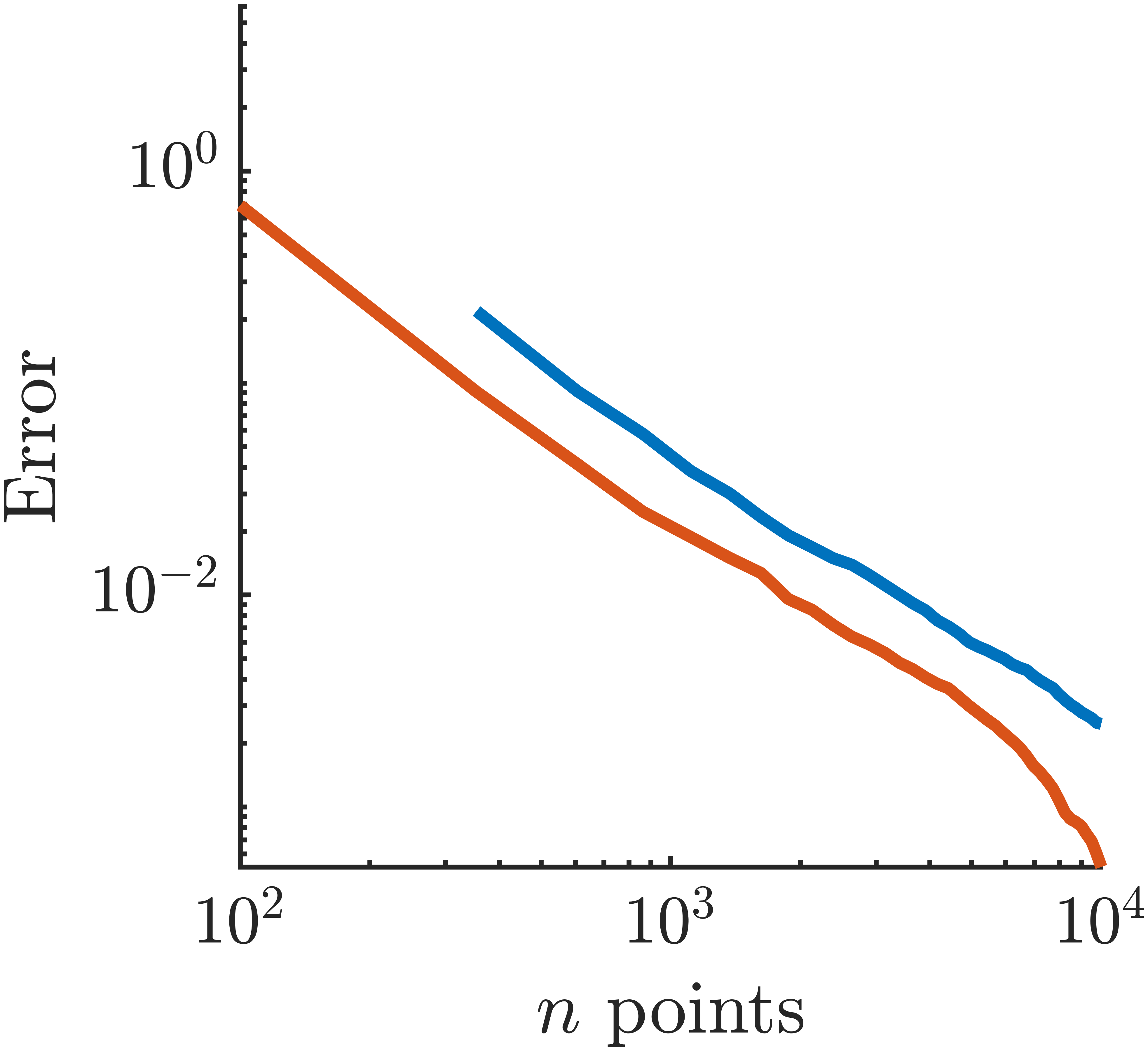}
    \end{subfigure}
    \vspace{-1em}
    \caption{Error analysis of the true geodesic length vs the length of the approximated geodesic as described in Theorem \ref{thm:conv-of-k-approx-lin-interp-cost-to-geodesic-cost}, in a log plot. The $x$ axis indicates the number of points and the $y$ axis indicates the absolute error in the geodesic approximation. The blue line uses points sampled uniformly from the ambient space, and the orange line uses points sampled from $p_\Psi$.}
    \label{fig:n-convergence}
\end{figure}
A visualization of the convergence described by Theorem \ref{thm:conv-of-k-approx-lin-interp-cost-to-geodesic-cost} is shown in Figure \ref{fig:n-convergence}. The statement in Theorem \ref{thm:conv-of-k-approx-lin-interp-cost-to-geodesic-cost} requires the data to be sampled from a uniform distribution, but we still numerically observe convergence even if the samples are not uniform. In fact, we observe faster convergence when the data is sampled from $p_\Psi$.
In Appendix~\ref{app:algorithm}, we provide details about the algorithm to compute generative distances and corresponding geodesics.

\begin{figure}[t]
\centering
\includegraphics[width=0.8\linewidth,trim=0 50 0 65, clip]{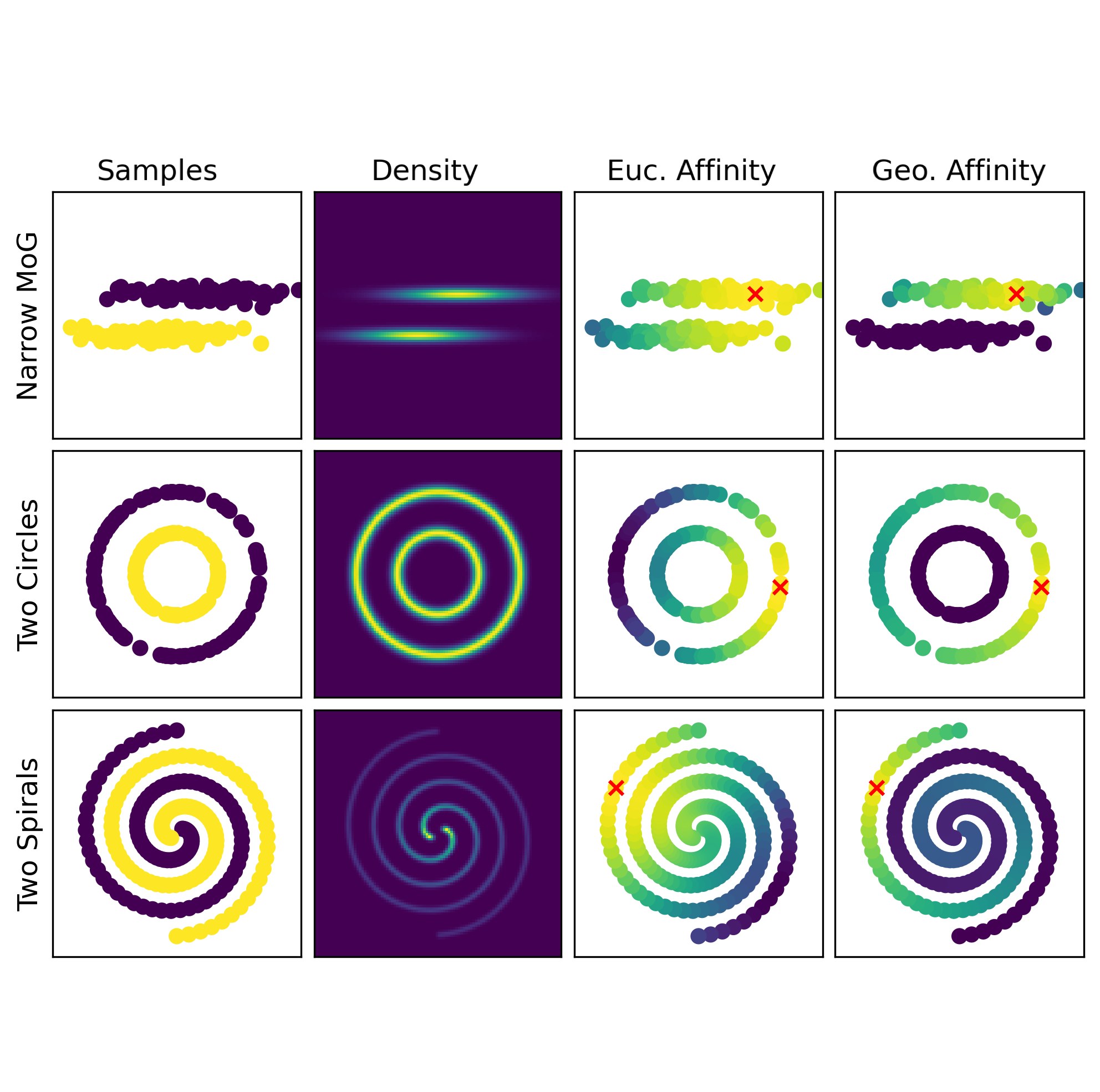}
\vspace{-3mm}
\caption{Euclidean and geodesic affinity visualization on various toy datasets. \textbf{First column} shows data samples, where color denotes the cluster each sample was generated from. \textbf{Second column} shows data density, where brighter color means higher density. \textbf{Third and fourth columns} show Euclidean and geodesic affinity, respectively, between the red cross and rest of the points. Brighter color means higher affinity.}
\label{fig:toy_viz}
\vspace{-3mm}
\end{figure}

\begin{figure}[t]
\centering
\includegraphics[width=.8\linewidth,trim=0 8 0 8, clip]{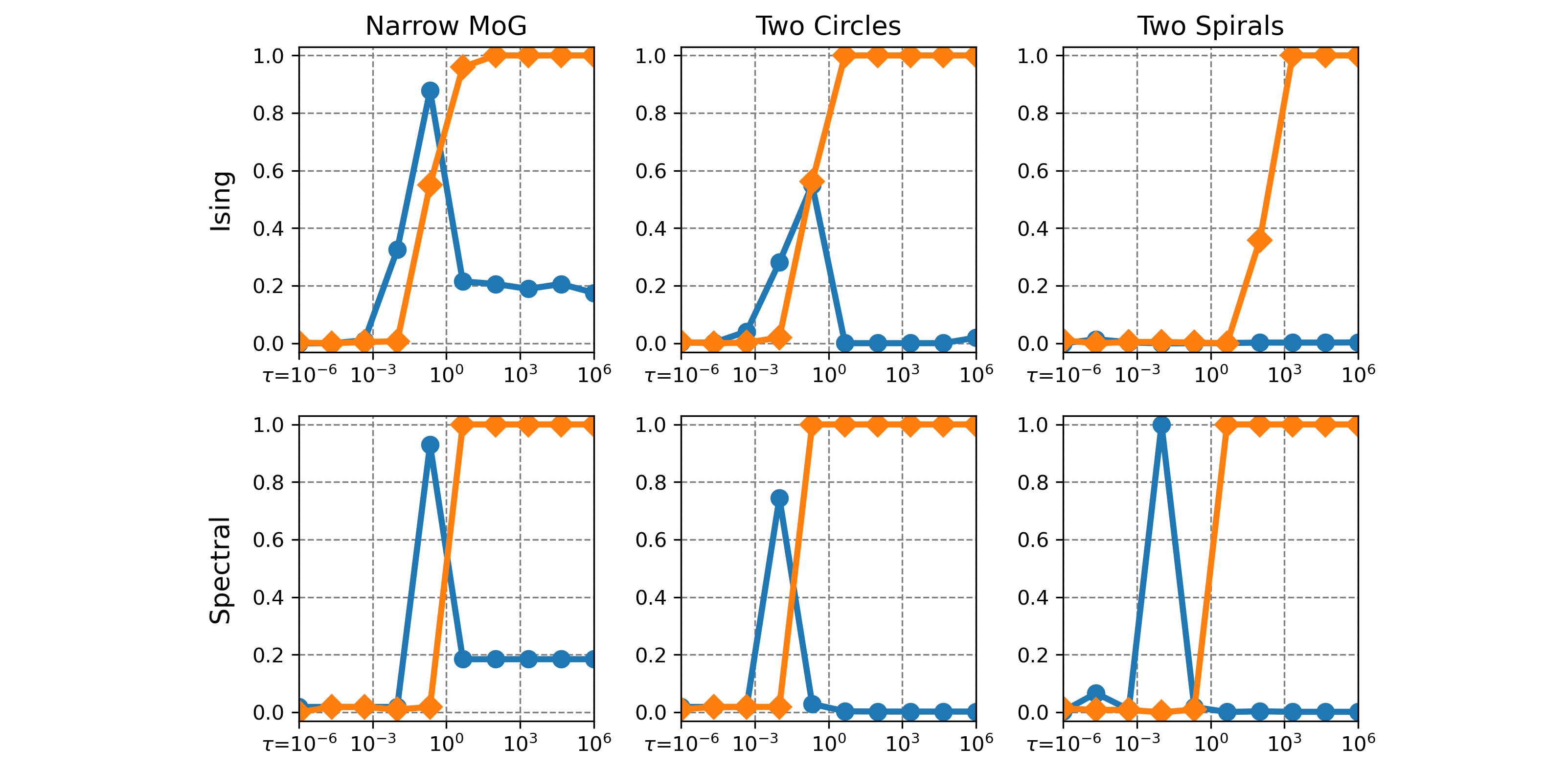}
\vspace{-4mm}
\caption{Clustering normalized mutual information (NMI) w.r.t. the temperature hyper-parameter $\tau$ using Euclidean and geodesic affinity matrices. Higher NMI is better. NMI = 0 means random cluster assignment and NMI = 1 means perfect clustering.}
\label{fig:toy_clustering}
\vspace{-3mm}
\end{figure}

{\section{Applications/Experiments}\label{sec:experiments}}

Detailed experiment settings can be found in Appendix \ref{append:settings}.

\subsection{Clustering with Geodesics} \label{sec:experiments_clustering}

In this section, we verify the effectiveness of the geodesic on three datasets: two parallel and axis-aligned MoG (Narror MoG), two concentric circles, and two spirals. Specifically, given a set of points $\{x_i\}$ from a dataset, we use the geodesic to construct an affinity matrix $A$ whose entries are defined as
    $A_{ij} \coloneqq \exp(-d_{\Psi,\lambda}(x_i,x_j) / \tau$
where $\tau$ is a temperature hyper-parameter to control the level of smoothness of $A$. Then the $(i,j)$-th entry of $A$ measures similarity between $x_i$ and $x_j$, with higher value indicating larger similarity under the Riemannian metric. We then use $A$ to run clustering algorithms. We consider two clustering methods: Ising kernel~\cite{kumagai2023ising} and spectral clustering. 
Figure \ref{fig:toy_viz} illustrates how geodesic affinity is able to provide robust clustering results. Indeed, compared to Euclidean affinity, we observe that geodesic affinity takes into account the geometric structure of the dataset.

In Figure \ref{fig:toy_clustering}, we see that clustering with geodesic affinity matrices out-perform clustering with Euclidean affinity on most of the datasets. Moreover, clustering with geodesic affinity is more robust to the choice of $\tau$. Specifically, Euclidean affinity achieves its peak performance only for a narrow subset $\tau \in (10^{-3},1)$ whereas geodesic affinity performs consistently well for all sufficiently large $\tau$.

\begin{figure}[t]
\centering
\includegraphics[width=0.8\linewidth,trim=0 0 0 5, clip]{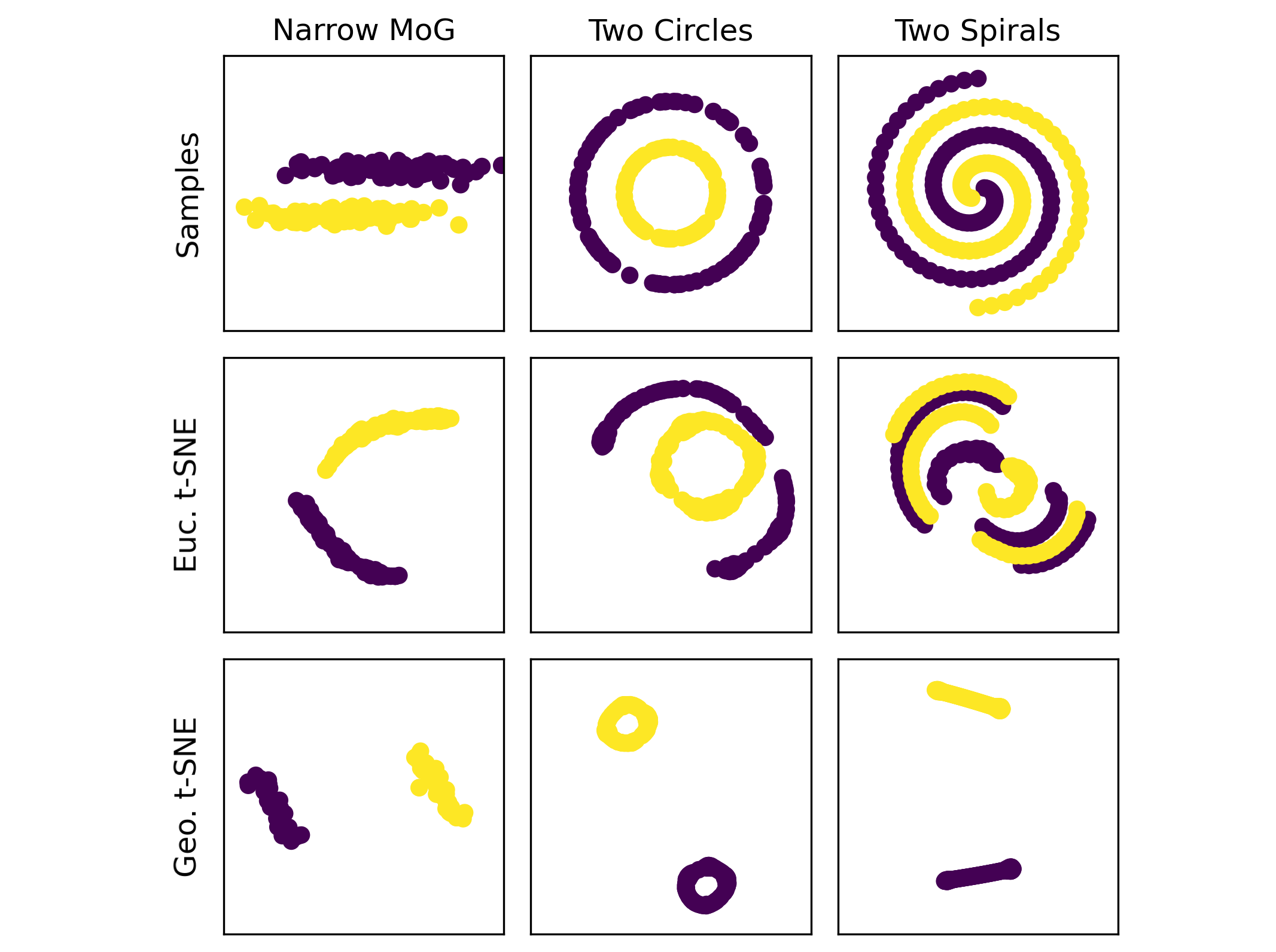}
\vspace{-3mm}
\caption{t-SNE with Euclidean and Riemannian metrics. Color denotes cluster label of each point.}
\label{fig:toy_tsne}
\vspace{-3mm}
\end{figure}

\subsection{Embedding with Geodesics} \label{sec:experiments_embedding}

We also visualize t-SNE embeddings computed w.r.t. the Euclidean metric and our Riemannian metric. Again, in Figure \ref{fig:toy_tsne}, we can observe that points on the same connected component, i.e., cluster, are nearby in the embedding space for the geodesic. In fact, t-SNE with the Riemannian metric always produces linearly separable embeddings, so we can achieve near-perfect classification with a linear classifier trained on the embeddings.

\subsection{Interpolating with Geodesics} \label{sec:experiments_interpolating}

Finally, we visualize geodesics approximated by our Riemmanian metric. In Figure \ref{fig:realnvp_inter}, we train a RealNVP \cite{dinh2017realnvp} on the two moons dataset, use the density estimated by RealNVP to calculate the geodesics. For baseline, we perform linear interpolation in the latent space of RealNVP. We observe that latent interpolation often produces convoluted paths whereas geodesics traverse the shortest path along the manifold.
\begin{figure}[t]
\centering
\includegraphics[width=.8\linewidth]{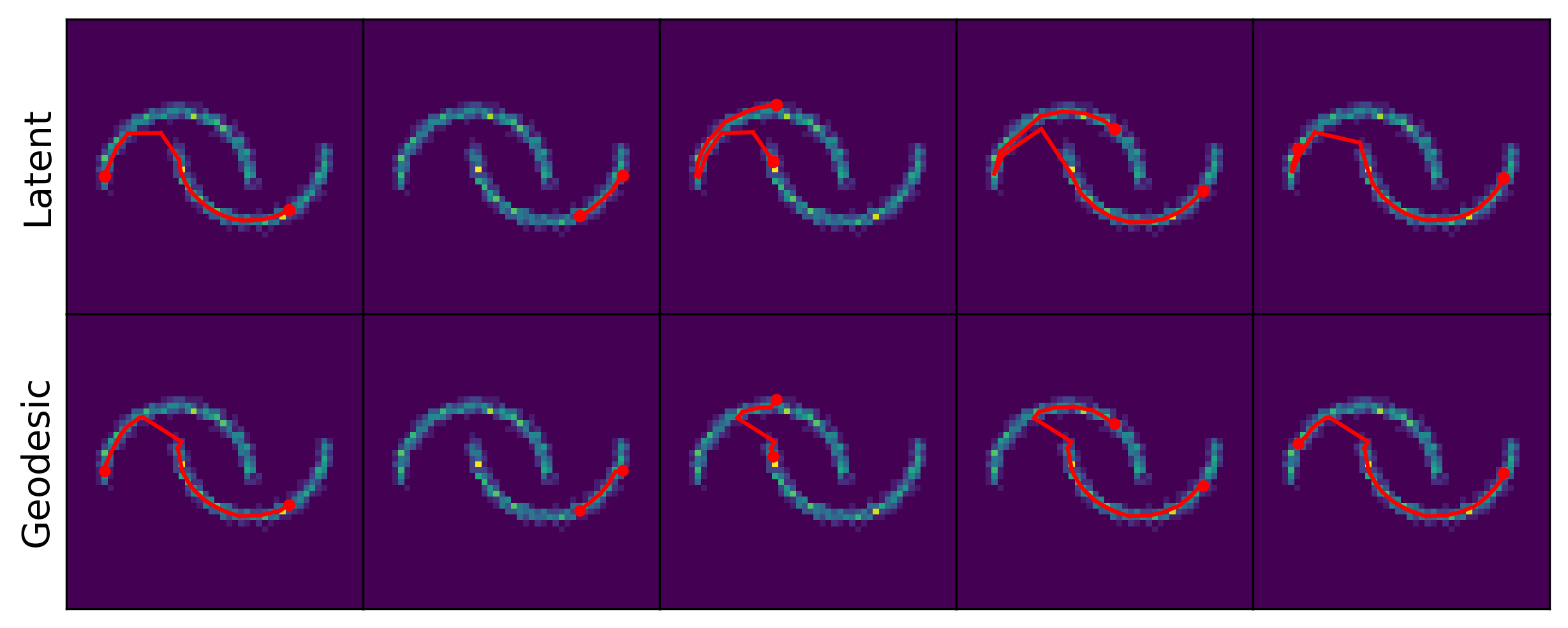}
\vspace{-4mm}
\caption{Interpolation on two moons. Each interpolation is illustrated by a red path.}
\label{fig:realnvp_inter}
\vspace{-5mm}
\end{figure}
\begin{figure}[b]
\centering
\vspace{-3mm}
\includegraphics[width=1.\linewidth,trim=0 0 0 5, clip]{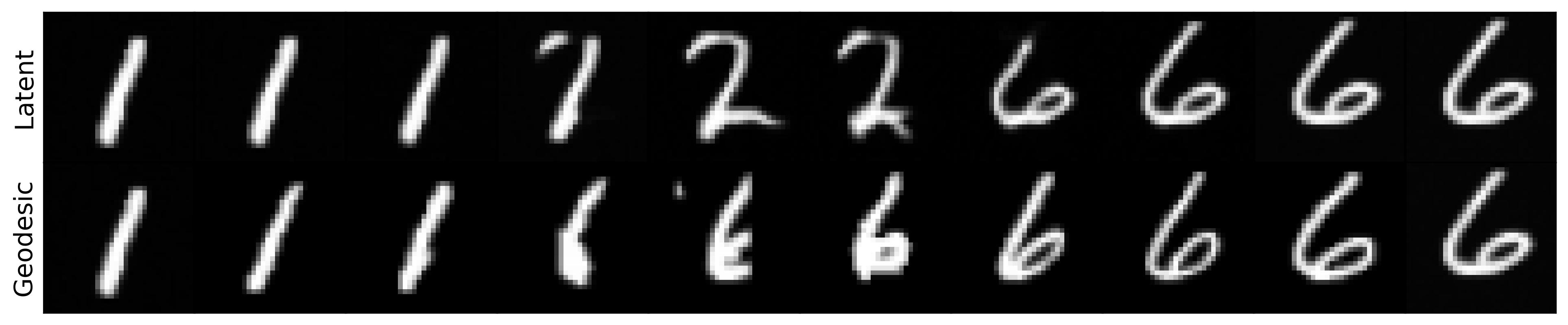} \\
\includegraphics[width=1.\linewidth]{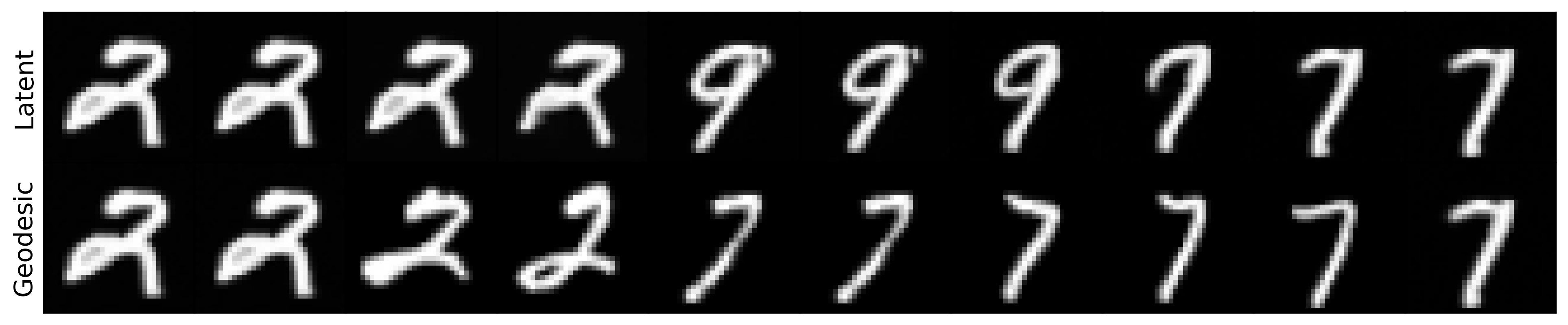} \\
\vspace{-3mm}
\caption{Interpolation on MNIST.}
\label{fig:mnist_inter}
%\vspace{-1mm}
\end{figure}
We also visualize geodesics between MNIST images in Figure \ref{fig:mnist_inter}, where the data density is approximated by a mixture of Gaussians with $2k$ components and variance $0.005^2$. For comparison, we also visualize interpolation in the latent space of a Consistency Trajectory Model \cite{kim2024ctm} which estimates diffusion probability flow ODE trajectories. We observe that the diffusion model often has sudden skips or artifacts when transitioning between digits of different identities. On the other hand, the geodesic always provides smooth transition between digits.

\section*{Statement of Contribution}
B. K. actively contributed to the discussion in the meetings, implemented and ran the experiments for Section 3.1-3.3 and wrote the corresponding text. 
M. P. developed the theory of the approach, wrote all theorems/lemmas and proofs, the text of Sections 1.1 and 2, the code behind Figures \ref{fig:eps-graph}, \ref{fig:comparison-of-parameters:a} and \ref{fig:n-convergence}, and wrote the pseudo-code for the geodesic computation.
J. C. Y. provided guidance and additional computation resources for the numerical experiments. 
E. S. setup and coordinated the team, helped to identify/conceptualize the metric, identified the applications in the experiments, reviewed the literature, wrote the Introduction section and proofread the manuscript.

\section*{Acknowledgements}

This work was partially supported by the National Research Foundation of Korea under Grant RS-2024-00336454. M. P. is partially supported by CAPITAL Services of Sioux Falls, South Dakota. E. S. is supported by the European Research Council (ERC)
under the European Union's Horizon 2020 research and innovation programme
(grant agreement n° 101021347, KeepOnLearning).

\bibliography{ref}
\bibliographystyle{icml2024}

\newpage
\appendix
\onecolumn

\section{Proofs}
\subsection{Proof of Theorem \ref{thm:conv-of-k-approx-lin-interp-cost-to-geodesic-cost}}
\label{sec:thm:conv-of-k-approx-lin-interp-cost-to-geodesic-cost}
The proof of Theorem \ref{thm:conv-of-k-approx-lin-interp-cost-to-geodesic-cost} has two parts. In the first part we prove that the linear interpolating cost (defined below) converges to $d_{\Psi,\lambda}$. Next, we show that the linear interpolating cost may be approximated again with $\hat L_{\Psi,\lambda}$ using a quadrature.

\begin{definition}[Linear Interpolating Cost]
    Let $x,y \in \mathcal X$ and $V_n(x,y)$ notate the set of all paths in $\mathcal{G}_n$ from $x$ to $y$ of the form 
    % \begin{align*}
    $
        V_n(x,y) \coloneqq \set{(v_0,v_1,\dots,v_{m-1},v_m) | m > 0, v_0 = x, v_m = y, v_1,\dots,v_{m-1} \in \mathcal X}.
        $
    % \end{align*}
    If $V_n(x,y)$ is nonempty, we say that 
    \emph{Linear Interpolating Cost} from $x,y$ is
    \begin{align}
        L_n(x,y) \coloneqq \min_{(v_0,\dots,v_m) \in V_n(x,y)} \sum_{i = 0}^{m-1} L(\overline{(v_i,v_{i+1})};\Psi,\lambda).
    \end{align}
    If $V_n(x,y)$ is empty, we say that $L_n(x,y) = +\infty$.
\end{definition}

Before proceeding to the Proof of Theorem \ref{thm:conv-of-k-approx-lin-interp-cost-to-geodesic-cost}, we first prove the following proposition about $L_n$.

\begin{proposition}[Convergence of Linear Interpolation Costs to Geodesic Cost]
    \label{prop:conv-of-linear-interp-cost}
    Let $\paren{\epsilon_n}_{i = 1,\dots}$ satisfy the decay rate assumption. Let $\mathcal X_n$ be $n$ realizations 
    of a uniform random variable $\textsc{x}$ over $\Omega$. Then the following results hold with probability 1.
    \begin{enumerate}
        \item For any $x,y \in \mathcal X_n$, $V_n(x,y)$ is non-empty.
        \item $\lim_{n \to \infty} L_n(x,y) = d_{\Psi,\lambda}(x,y)$.
        \item For any sequence of optimal paths $l_{\bv^{(n)}} \in \argmin L_n(x,y)$ and subsequence of $\paren{l_{\bv^{(n)}}}_{n = 1,\dots}$ has a further subsequence of linear paths that converges uniformly to a limit path $\gamma \in \argmin L$ and of the discrete paths in the Hausdorff sense to $S_\gamma$ an optimal path. Further, if the path is unique (up to reparameterization), then $\paren{l_{\bv^{(n)}}}_{n = 1,\dots}$ converges in Hausdorff distance to $S_\gamma$.
    \end{enumerate}
\end{proposition}

\begin{proof}
    The first point follows from the paragraph at the end of Pg. 5 \cite{davis2019approximating}.

    For all $x$ and $\lambda > 0$, $\frac{p_0 + \lambda}{p_\Psi(x) + \lambda}$ is positive, hence 
    Lemma 2.6 \cite{davis2019approximating} applies by taking $M(x) \coloneqq \frac{p_0 + \lambda}{p_\Psi(x) + \lambda}$. 
    Hence, (Lip), (Hilb), (TrIneq) and (Pythag) all apply for $\alpha > 1$. Therefore the full consequence 
    of Theorem 2.8 \cite{davis2019approximating} apply. 
    
    The second point follows from equation 
    $\lim_{n \to \infty} \min_{\gamma \in \Omega^l_n(x,t)}L_n(x,y) = \min_{\gamma \in \Omega(x,y)}L(\gamma; \Psi,\lambda)$. 
    But, we have that $\min_{\gamma \in \Omega(x,y)}L(\gamma; \Psi,\lambda) = d_{\Psi,\lambda}(x,y)$ from the work above.

    Finally, the third point follows from the final three sentences in the conclusion of Theorem 2.8 \cite{davis2019approximating}.
\end{proof}

Next we prove two lemmas that show that $L_n$ and $\hat L_{n,K}$ are close, and become closer as $n,K \to \infty$.

\begin{lemma}[Approximating Edge Weights]
    \label{lem:approx-edge-weights}
    Let $\cG_n$ be a $\lambda$ weighted $\epsilon$ graph with edge $(x,y)$. Then for any $\delta > 0$, 
    there is a $K \in \bbN$ so that %\map{SHOULD IT BE BIG K IN THE DENOMINATOR INSTEAD OF SMALL k?}
    \begin{align}
        \abs{L(\overline{(x,y)};\Psi,\lambda) - \frac{\norm{y - x}_2}{K}\sum_{i = 0}^{K-1}\frac{p_0 + \lambda}{p_\Psi\paren{\paren{1 - \frac iK}x + \frac iK y} + \lambda}} < \delta
    \end{align}
\end{lemma}

\begin{proof}
    For $x,y \in \Omega$, the unit-speed parameterization of the straight-line path connecting 
    $x$ to $y$ is given by
    \begin{align*}
        \gamma(t) \coloneqq \paren{1 - \frac{t}{\norm{y-x}_2}}x + \frac{t}{\norm{y-x}_2}y, \quad \dot \gamma(t) = \frac{y - x}{\norm{y-x}_2}.
    \end{align*}
    Hence, we may write
    \begin{align}
        L(\overline{(x,y)};\Psi,\lambda) &= \int^{\norm{y-x}_2}_{0}f\paren{\paren{1 - \frac{t}{\norm{y-x}_2}}x + \frac{t}{\norm{y-x}_2}y, \frac{y - x}{\norm{y-x}_2};\lambda}dt \nonumber\\
        &= \int^{\norm{y-x}_2}_0\frac{p_0 + \lambda}{p_\Psi\paren{\paren{1 - \frac{t}{\norm{y-x}_2}}x + \frac{t}{\norm{y-x}_2}y} + \lambda} dt. \label{eqn:integral-expansion-f}
    \end{align}
    For $\Delta t \coloneqq \frac{\norm{y - x}_2}{K}$, the $K$-piece left Riemann sum of the r.h.s. 
    of \ref{eqn:integral-expansion-f} is %\map{SHOULD IT BE BIG K IN THE DENOMINATOR INSTEAD OF SMALL k?}
    \begin{align}
        \sum_{i = 0}^{K-1}f(\gamma(\Delta t i), \dot \gamma(\Delta t i);\Psi,\lambda)\Delta t = \frac{\norm{y - x}_2}{K}\sum_{i = 0}^{K-1}\frac{p_0 + \lambda}{p_\Psi\paren{\paren{1 - \frac iK}x + \frac iK y} + \lambda}
    \end{align}
    The quantity in the integrand of Eqn. \ref{eqn:integral-expansion-f} is continuous as a 
    function of $t$, and so from standard calculus results there is some $K$ so that %\map{SHOULD IT BE BIG K IN THE DENOMINATOR INSTEAD OF SMALL k?}
    \begin{align}
        \abs{L(\overline{(x,y)};\Psi,\lambda) - \frac{\norm{y - x}_2}{K}\sum_{i = 0}^{K-1}\frac{p_0 + \lambda}{p_\Psi\paren{\paren{1 - \frac iK}x + \frac iK y} + \lambda}} < \delta.
    \end{align}
\end{proof}

\begin{lemma}[Approximating Linear Interpolating Cost]
    \label{lem:approx-lin-interp-cost}
    Let $\cX$ and $\lambda$ be given, then for any $\delta$ there is a $K \in \bbN$ such that for all $x,y\in \cX$,
    \begin{align}
        \abs{L_n(x,y) - \hat L_{n,K}(x,y)} < \delta.
    \end{align}
\end{lemma}

\begin{proof}
    Let $M$ be the length (in number of edges) between all $L_n$ minimal paths connecting $x,y$ for any 
    two $x,y$. Then, we may apply Lemma Approximating Edge Weights to get a $K$ such that the edge weight 
    error is no more than $\delta/M$. Then, we have for that choice of $K$ and any path $(v_0,\dots,v_m)$ that %\map{SHOULD IT BE BIG K IN THE DENOMINATOR INSTEAD OF SMALL k?}
    \begin{align*}
        &\abs{\sum_{i = 0}^{m-1} L(\overline{(v_i,v_{i+1})};\Psi,\lambda) - \frac{\norm{y - x}_2}{K}\sum^{m-1}_{i = 0} \sum_{j = 0}^{K-1}\frac{p_0 + \lambda}{p_\Psi\paren{\paren{1 - \frac jK}v_i + \frac jK v_{i+1}} + \lambda}}\\
        &= \abs{\sum_{i = 0}^{m-1} \paren{L(\overline{(v_i,v_{i+1})};\Psi,\lambda) - \frac{\norm{y - x}_2}{K} \sum_{j = 0}^{K-1}\frac{p_0 + \lambda}{p_\Psi\paren{\paren{1 - \frac jK}v_i + \frac jK v_{i+1}} + \lambda}}}\\
        &\leq \sum_{i = 0}^{m-1} \abs{L(\overline{(v_i,v_{i+1})};\Psi,\lambda) - \frac{\norm{y - x}_2}{K} \sum_{j = 0}^{K-1}\frac{p_0 + \lambda}{p_\Psi\paren{\paren{1 - \frac jK}v_i + \frac jK v_{i+1}} + \lambda}}\\
        &\leq \sum_{i = 0}^{m-1} \frac\delta M = \frac mM\delta \leq \delta
    \end{align*}
    This bound holds for all paths, hence it holds for the minimal paths too. Let 
    $\bv^* \coloneqq (v_0^*,\dots,v_m^*)$ be an $L_n$ minimizing path from $x,y$, then %\map{SHOULD IT BE BIG K IN THE DENOMINATOR INSTEAD OF SMALL k?}
    \begin{align*}
        &\abs{L_n(x,y) - \hat L_{n,K}(x,y)}\\
        &\leq \abs{\sum_{i = 0}^{m-1} L(\overline{(v_i^*,v^*_{i+1})};\Psi,\lambda) - \frac{\norm{y - x}_2}{K}\sum^{m-1}_{i = 0} \sum_{j = 0}^{K-1}\frac{p_0 + \lambda}{p_\Psi\paren{\paren{1 - \frac jK}v^*_i + \frac jK v^*_{i+1}} + \lambda}}\\
        &\leq \delta.
    \end{align*}
\end{proof}

Finally, we present the final proof of Theorem \ref{thm:conv-of-k-approx-lin-interp-cost-to-geodesic-cost}. 

\begin{proof}
    The vertices and edges of $\cG_n$ and $\hat \cG_{n,K_n}$ are the same, so the existence of a path connecting $x$ and $y$ in $\cX$ follow from the same arguments as in Prop. \ref{prop:conv-of-linear-interp-cost}.

    Point 2 follows from combining point 2 of Prop. \ref{prop:conv-of-linear-interp-cost} with Lemma \ref{lem:approx-lin-interp-cost}.

    For $\lambda > 0$, $\frac{p_0 + \lambda}{p_\Psi(x) + \lambda}$ is smooth on compact set, and so has uniformly bounded gradients. Hence, for any $\delta$, there is a uniform $K$ that can be so that for all $\overline{(x,y)}$ such that $x,y \in \Omega$. Hence, as $n\to\infty$,
    \begin{align}
        \abs{L(e;\Psi,\lambda) - \frac{\norm{y - x}_2}{K}\sum_{i = 0}^{K-1}\frac{p_0 + \lambda}{p_\Psi\paren{\paren{1 - \frac iK}x + \frac iK y} + \lambda}} \to 0
    \end{align}
    uniformly in $e$. Hence, for any sequence $\paren{l_{\bv^{(n)}}}_{n = 1,\dots}$ of paths uniformly bounded in length, %\map{SHOULD IT BE BIG K IN THE DENOMINATOR INSTEAD OF SMALL k?}
    \begin{align}
        \label{eqn:unif-conv-of-arc-lenghs}
        \lim_{n \to \infty}\abs{\sum_{i = 0}^{m_n-1}\paren{L(\overline{(v_i,v_{i+1})};\Psi,\lambda) - \frac{\norm{y - x}_2}{K} \sum_{j = 0}^{K_n-1}\frac{p_0 + \lambda}{p_\Psi\paren{\paren{1 - \frac jK}v_i + \frac jK v_{i+1}} + \lambda}}} = 0,
    \end{align}
    where $l_{\bv^{(n)}} = (v_0,\dots,v_{m_n})$. Put simply, for a sequence of finite-length paths in $\cG$, the costs of $L_n$ and $\hat L_{n,K_n}$ converge to each other as $n\to\infty$.

    Now, let $l_{\bv^{(n)}}\in \argmin \hat L_{n,K_n}(x,y)$. We have that $\frac{p_0 + \lambda}{p_\Psi(x) + \lambda}$ is bounded from below, hence $l_{\bv^{(n)}}$ are all uniformly bounded in length. Hence, Eqn. \ref{eqn:unif-conv-of-arc-lenghs} applies, and so $\liminf_{n \to \infty} L_n(l_{\bv^{(n)}}) \leq \lim_{n \to \infty} \hat L_{n,K_n}(l_{\bv^{(n)}})$. The same argument holds reversing the roles of $L_n$ and $\hat L_{n,K_n}$, hence any minimizing sequence of $L_n$ is one of $\hat L_{n,K_n}$ and vice-versa.

    Now suppose that $(l_{\bv^{(n)}})_{n = 1,\dots} \in \argmin \hat L_{n,K_n}(x,y)$ is a sequence of optimal paths. It is also a sequence of optimal paths of $L_n$. Thus, any subsequence of $(l_{\bv^{(n)}})$ has a further subsequence so that Theorem \ref{prop:conv-of-linear-interp-cost} point 3 applies, and so has a further subsequence that converges in Hausdorff distance to $S_\gamma$.
\end{proof}

{\section{Algorithm}\label{app:algorithm}}

Algorithm \ref{alg:hat-g} describes the strategy to compute $\hat \cG_{n,K}$. From $\hat \cG_{n,K}$, we can compute the minimal path between two end points $x$ and $y$ and its length. This requires us to solve a standard graph theorem problem, which can be solved with Dijkstra's algorithm, see e.g. \cite{crauser1998parallelization,ortega2013new}. A GPU accelerated python implementation of the single source shortest path (SSSP) algorithm can be found in the RAPIDs library~\cite{rapids2023}.

\begin{algorithm}
\caption{Compute $\hat \cG_{n,K}$} \label{alg:hat-g}
\begin{algorithmic}[1]
\REQUIRE edge $\epsilon, X \subset \cX, \lambda, p$

\text{\# Compute $\cG_n$ with vertices $V$, edges $E$}
\STATE $V,E \gets X,\emptyset$
\FOR{$(x,y) \in X \times X$, $x \neq y$}
    \IF{$d_2(x,y) \leq \epsilon$}
        \STATE $E$.add$(x,y)$
    \ENDIF
\ENDFOR

\text{\# Compute edge weights, $W\colon E \to \R$}
\STATE $W$ reset
\FOR{edge $(x,y) \in E$}
    \STATE $W((x,y)) \gets \text{Integral Quadrature for } L_g(\ell_{x,y})$
\ENDFOR
\STATE \textbf{return} $\hat \cG_{n,K} = (\cG_n,W)$.
\end{algorithmic}
\end{algorithm}

\section{Experiment Settings} \label{append:settings}

\subsection{Section \ref{sec:deepgenerative:approx-and-conv}}

In Figure \ref{fig:n-convergence}, we used a $p_\Psi$ given by
\begin{align}
    p_\Psi(x) \coloneqq \frac{e^{- 10 \abs{\frac34 - \norm{x}_2}}}{Z}.
\end{align}
defined over $[0,1]^2$, where $Z$ is a normalization constant. We used $K = 10$, $\lambda = 0.01$ $\epsilon = 0.158$. The values of $n$ chosen ranged from $n = 100$ to $n = 10,000$. The geodesic error was measured between the points $(-1,0)$ and $(1,0)$. For each value of $n$, we computed the associated geodesic error 20 times. Figure \ref{fig:n-convergence} reports the average error across the 20 trials.

\subsection{Section \ref{sec:experiments_clustering}}

We use $n=200$, $K=10$, $\lambda = 10^{-8}$, and $\epsilon=10$ to compute $\hat{\mathcal{G}}_{n,K}$. We use spectral clustering in the \texttt{scikit-learn} Python library, and Ising clustering from \url{https://github.com/kumagaimasahito/Ising-based_Kernel_Clustering}. We use the default clustering hyper-parameters. For evaluation, we use normalized mutual information also in \texttt{scikit-learn}. 

\subsection{Section \ref{sec:experiments_embedding}}

We use $n=200$, $K=10$, $\lambda = 10^{-8}$, and $\epsilon=10$ to compute $\hat{\mathcal{G}}_{n,K}$. We use t-SNE in the \texttt{scikit-learn} Python library.

\subsection{Section \ref{sec:experiments_interpolating}}

\textbf{Two Moons.} We train a RealNVP model using the code in \url{https://github.com/xqding/RealNVP}. We use $n=1024$, $K=10$, $\lambda = 10^{-8}$, and $\epsilon=10$ to compute $\hat{\mathcal{G}}_{n,K}$. For latent interpolation, given two points, we map them to latent vectors using the RealNVP model, linearly interpolate the latent vectors, and then push the latent vectors to data space by again using the RealNVP model.

\textbf{MNIST.} We use $n=2000$, $K=10$, $\lambda = 0.005$, and adaptively choose $\epsilon$ for each vertex in $\hat{\mathcal{G}}_{n,K}$ such that it always has at least $5$ neighbors. For the baseline, we train a Consistency Trajectory Model on MNIST using the code from \url{https://github.com/sony/ctm}, and interpolate images by linearly interpolating latent vectors and pushing the vectors through the model.

\end{document}